\documentclass{article} 

\usepackage{url}

\usepackage{fullpage}
\usepackage{times}
\usepackage{graphicx} 
\usepackage{pgf}
\usepackage{subfigure}
\usepackage{enumitem}

\usepackage{algorithm}
\usepackage{algorithmic}

\usepackage{verbatim}
\usepackage{blindtext}
\usepackage{amsthm}
\usepackage{amsmath}
\usepackage{amssymb}
\newtheorem{theorem}{Theorem}[section]
\newtheorem{claim}{Claim}[section]
\newtheorem{definition}{Definition}[section]
\newtheorem{lemma}[theorem]{Lemma}

\newcommand{\R}{\mathbb{R}}
\newcommand{\cum}{\kappa}
\newcommand{\E}{\mathbb{E}}
\newcommand{\N}{\mathcal{N}}
\newcommand{\alg}{\mbox{RCA }}

\newenvironment{fact}[1][Fact]{\begin{trivlist}
\item[\hskip \labelsep {\bfseries #1}]}{\end{trivlist}}

\usepackage{hyperref}

\title{Rich Component Analysis}
\author{Rong Ge \\ Microsoft Research, New England \\ rongge@microsoft.com \and James Zou \\ Microsoft Research, New England \\ jazo@microsoft.com}
\date{}
\begin{document}

\maketitle

\begin{abstract}

In many settings, we have multiple data sets (also called views) that capture different and overlapping aspects of the same phenomenon. We are often interested in finding patterns that are unique to one or to a subset of the views. For example, we might have one set of molecular observations and one set of physiological observations on the same group of individuals, and we want to quantify molecular patterns that are uncorrelated with physiology. Despite being a common problem, this is highly challenging when the correlations come from complex distributions. In this paper, we develop the general framework of Rich Component Analysis (RCA) to model settings where the observations from different views are driven by different sets of latent components, and each component can be a complex, high-dimensional distribution. We introduce algorithms based on cumulant extraction that provably learn each of the components without having to model the other components. We show how to integrate RCA with stochastic gradient descent into a meta-algorithm for learning general models, and demonstrate substantial improvement in accuracy on several synthetic and real datasets in both supervised and unsupervised tasks.  Our method makes it possible to learn latent variable models when we don't have samples from the true model but only samples after complex perturbations.
\end{abstract}

\section{Introduction}
A hallmark of modern data deluge is the prevalence of complex data that capture different aspects of some common phenomena. For example, for a set of patients, it's common to have multiple modalities of molecular measurements for each individual (gene expression, genotyping, etc.) as well as physiological attributes. Each set of measurements corresponds to a \emph{view} on the samples. The complexity and the heterogeneity of the data is such that it's often not feasible to build a joint model for all the data. Moreover, if we are particularly interested in one aspect of the problem (e.g. patterns that are specific to a subset of genes that are not shared across all genes), it would be wasteful of computational and modeling resources to model the interactions across all the data. 

More concretely, suppose we have two sets (views) of data, $U$ and $V$, on a common collection of samples. We model this as $U = S_1 + S_2$ and $V = AS_2 + S_3$, where $S_1$ captures the latent component specific to $U$, $S_3$ is specific to $V$, and $S_2$ is common to both $U$ and $V$ and is related in the two views by an unknown linear transformation $A$. Each component $S_i$ can be a complex, high-dimensional distribution. The observed samples from $U$ and $V$ are component-wise linear combinations of the unobserved samples from $S_i$. To model all the data, we would need to jointly model all three $S_i$, which can have prohibitive sample/computation complexity and also prone to model misspecification. Ideally, if we are only interested in the component that's unique to the first view,  we would simply write down a model for $S_1$ without making any parametric assumptions about $S_2$ and $S_3$, except that they are independent.       

In this paper, we develop a general framework of Rich Component Analysis (RCA) to explore such multi-component, multi-view datasets. Our framework allows for learning an arbitrarily complex model of a specific component of the data, $S_i$, without having to make parametric assumptions about other components $S_j$. This allows the analyst to focus on the most salient aspect of data analysis. The main conceptual contribution is the development of new algorithms to 
learn parameters of complex distributions without any samples from that distribution. In the two-view example, we do not observe samples from our model of interest, $S_1$. Instead the observations from $U$ are compositions of true samples from $S_1$ with complex signal from another process $S_2$ which is shared with $V$. Our approach performs consistent parameter estimation of $S_1$ without modeling $S_2$, $S_3$.  

\paragraph{Outline.} \alg consists of two stages: 1) from the observed data, extract all the cumulants of the component that we want to model; 2) using the cumulants, perform method-of-moments or maximum likelihood estimation (MLE) of model parameters via polynomial approximations to gradient descent. We introduce the relevant properties of cumulants and tensors in Section~\ref{sec:prelim}. In Section~\ref{sec:cumulant}, we  develop the formal models for Rich Component Analysis (RCA) and the cumulant extraction algorithms. We discuss how \alg differs from existing models. Section~\ref{sec:apply} shows how to integrate the extracted cumulants with method-of-moments or stochastic gradient descent for MLE inference. We show the performance gains of \alg in Section~\ref{sec:expt}. All the proofs are in the Appendix.

\section{Preliminaries}
\label{sec:prelim}

In this section we introduce the basics of cumulants. For more information please refer to Appendix~\ref{app:tensor}. Cumulants provide an alternative way to describe the correlations of random variables. Unlike moments, cumulants have the nice property that the cumulant of sum of independent random variables equals to the sum of cumulants. For a random variable $X \in \R$ the cumulant is defined to be the coefficients of the cumulant generating function $\log \E[e^{tX}]$.

We can also define cross-cumulants which are cumulants for different variables (e.g. covariance). For $n$ variables $X_1,...,X_n$, their cross-cumulant can be computed using the following formula:
$$\cum_t(X_1,...,X_t)  = \sum_{\pi}(|\pi|-1)!(-1)^{|\pi|-1}\prod_{B\in \pi}\E[\prod_{i\in B} X_i].$$
In this formula, $\pi$ is enumerated over all partitions of $[t]$, $|\pi|$ is the number of parts and $B$ runs  through the list of parts. We also use $\cum_t(X) \equiv \cum_t(X,...,X)$ when it's the same random variable.

We can similarly define cumulants for multivariate distributions. For random vector $X\in\R^d$, the $t$-th order cumulant (and $t$-th order moment) is an object in $\R^{d^t}$ (a $t$-th order tensor). The $(i_1,...,i_t)$-th coordinate of cumulant tensor is  $\cum_t(X_{i_1},X_{i_2},...,X_{i_t})$. We often {\em unfold} tensors into matrices. Tensor $T \in \R^{d^t}$ unfolds into matrix $M=unfold(T)\in \R^{d^{t-1}\times d}$:
$
M_{(i_1,...,i_{t-1}), i_t} = T_{i_1,...,i_t}.
$
Cumulants have several nice properties that we summarize below.

\begin{fact}\label{lem:cumulant}
Suppose $X_1,...,X_t$ are random variables in $\R^d$. The $t$-th order cumulant $\cum_t(X_1,...,X_t)$ is a tensor in $\R^{d^t}$ that have the following properties:
\begin{enumerate}[topsep=0pt,itemsep=0pt]
\item(Independence) If $(X_1,...,X_t)$ and $(Y_1,...,Y_t)$ are independent, then $\cum_t(X_1+Y_1,...,X_t+Y_t) = \cum_t(X_1,...,X_t) + \cum_t(Y_1,...,Y_t)$.
\item(Linearity) $\cum_t(c_1 X_1,...,c_t X_t) = c_1c_2\cdots c_t \cum_t(X_1,...,X_t)$, more generally we can apply arbitrary linear transformations to multi-variate cumulants (see Appendix~\ref{app:tensor}).
\item(Computation) The cumulant $\kappa_t(X_1,...,X_t)$can be computed in $O((td)^t)$ time.
\end{enumerate}
\end{fact}

The second order cross-cumulant, $\cum_2(X, Y)$ is equal to the covariance $\E[(X - \E[X])(Y - \E[Y])]$. Higher cumulants measures higher-order correlations and also provide a measure of the deviation from Gaussianity--all 3rd and higher order cumulants of Gaussian random variables are zero.

\section{Rich Component Analysis}
\label{sec:cumulant}

In this section, we show how to use cumulant to disentangle complex latent components. The key ideas and applications of \alg are captured in the contrastive learning setting when there are two views. We introduce this model next and then show how to extend it to general settings.

\subsection{RCA for contrastive learning}
\label{sec:contrast}

Recall the example in the introduction where we have two views of the data, formally,
\begin{equation}
U  = S_1 + S_2, V  = AS_2 + S_3.
\end{equation}
Here, $S_1,S_2,S_3 \in \R^d$ are independent random variables that can have complicated distributions; $A\in \R^{d\times d}$ is an unknown linear transformation. The observations consist of pairs of samples $(u, v)$. Each pair is generated by drawing independent samples $s_i \sim S_i, i = 1,2,3$ and adding these samples component-wise to obtain $u = s_1 + s_2$ and $v = As_2 + s_3$. Note that the same $s_2$ shows up in both $u$ and $v$, introducing correlation between the two views. We are interested in learning properties about $S_i$, for example learning its maximum likelihood (MLE) parameters. For concreteness, we focus our discussion on learning $S_1$ although our techniques also apply to $S_2$ and $S_3$. We don't have any samples from $S_1$. The observations of $U$ involves a potentially complicated perturbation by $S_2$. Our hope is to remove this perturbation by utilizing the second view $V$, and we would like to do this without assuming a particular model for $S_2$ or $S_3$. 

Note that the problem is inherently under-determined: it is impossible to find the means of $S_1$, $S_2$, $S_3$ without any additional information. This is in some sense the only ambiguity, as we will see if we know the mean of one distribution it is possible to extract all order cumulants of $S_1,S_2,S_3$. For simplicity throughout this section we assume the means of $S_1,S_2,S_3$ are 0 (given the mean of any of $S_1,S_2,S_3$, we can always use the means of $U$ and $V$ to compute the means of other distributions, and shift them to have mean 0).

\paragraph{Determining linear transformation}

First we can find $A$ by the following formula:
\begin{equation}
A^\top = unfold(\cum_4(V,U,U,U))^\dag unfold(\cum_4(V,U,U,V)).\label{eq:findA}
\end{equation}

\begin{lemma}
\label{lem:contrastivelinear}
Suppose the unfolding of the 4-th order cumulant $unfold(\cum_4(AS_2,S_2,S_2,S_2))$ has full rank, given the exact cumulants $\cum_4(V,U,U,U)$ and $\cum_4(V,U,U,V)$, the above algorithm finds the correct linear transformation $A$ in time $O(d^5)$.
\end{lemma}

Intuitively, since only $S_2$ appears in both $U$ and $V$, the cross-cumulants $\cum_4(V,U,U,U)$ and $\cum_4(V,U,U,V)$ depend only on $S_2$. Also, by linearity of cumulants we must have $unfold(\cum_4(V,U,U,V)) = unfold(\cum_4(V,U,U,U))A^\top$ (see Appendix~\ref{app:contrastive}). 
In the lemma we could have used third order cumulants, however for many distributions (e.g. all symmetric distributions) the third order cumulant is 0. Most distributions satisfy the condition that $unfold(\cum_4(AS_2,S_2,S_2,S_2))$ is full rank, the only natural distribution that does not satisfy this constraint is the Gaussian distribution (where $\cum_4$ is 0). 

\paragraph{Extracting cumulants}

Even when the linear transformation $A$ is known, in most cases it is still information theoretically impossible to find the {\em values} of the samples $s_1, s_2, s_3$ as we only have two views. However, we can still hope to learn useful information about the {\em distributions} $S_1,S_2,S_3$. In particular, we derived the following formulas to estimate the cumulants of the distributions:
\begin{eqnarray}
\cum_t(S_1) & = &  \cum_t(U) - \cum_t(U,U,...,U,A^{-1}V), \label{eq:s1simple}\\
\cum_t(S_2) & = & \cum_t(U,U,...,U,A^{-1}V),\label{eq:s2simple}\\
\cum_t(S_3) & = & \cum_t(V) - \cum_t(AU,V,V,...,V).\label{eq:s3simple}
\end{eqnarray}
\begin{theorem}
\label{lem:cumulantsimple}
For all $t > 1$, Equations (\ref{eq:s1simple})-(\ref{eq:s3simple}) compute the $t$-th order cumulants for $S_1,S_2,S_3$ in time $O((td)^{t+2})$
\end{theorem}

Proof of Theorem~\ref{lem:cumulantsimple} relies on the fact that since only $S_2$ appears in both $U$ and $V$, the cross-cumulant $\cum_t(U,U,...,U,A^{-1}V)$ captures the cumulant of $S_2$. Moreover, by independence, $\cum_t(U) = \cum_t(S_1)+\cum_t(S_2)$, so we can recover $\cum_t(S_1)$ by subtracting off the estimated $\cum(S_2)$ (and similarly for $\cum_t(S_3)$). When the dimension of $U$ is smaller than the dimension of $V$ and $A\in \R^{d_V\times d_U}$ has full column rank, the above formula with pseudo-inverse $A^\dag$ in place of $A^{-1}$ still recovers all cumulants. In Appendix~\ref{app:contrastive}, we prove that both the formulas for computing $A$ and for extracting the cumulants are robust to noise. In particular, we give the sample complexity for learning $A$ and $\cum_t(S_1)$ from samples of $U$ and $V$, both are polynomial in relevant quantities. 

Given $\cum_t(S_1)$, we can use standard algorithms to compute moments of $S_1$. Many learning algorithms are based on method-of-moments and can be directly applied (see Section~\ref{sec:mom}). Other optimization-based algorithms can also be adapted (Section~\ref{sec:gradient}). 


\subsection{General model of Rich Component Analysis}

We can extend the cumulant extraction algorithm in contrastive learning to general settings with more views and components. The ideas are very similar, but the algorithm is more technical in order to keep track of all the components. We present the intuition and the main results here and defer the details to Appendix~\ref{app:general}.    
 Consider a set of observations $U_1, U_2, \ldots, U_k \in \R^d$, each is linearly related to a subset of variables $S_1,S_2, \ldots, S_p \in \R^d$, the variable $S_j$ appears in a subset $Q_j \subset [k]$ of the observations. That is,
 
\begin{equation}
\forall i\in [k] \quad U_i = \sum_{j=1}^p A^{(i,j)} S_j, \label{eq:model}
\end{equation}
where $A^{(i,j)}\in \R^{d\times d}$ are unknown linear transformations, and $A^{(i,j)} = 0$ if $i \not\in Q_j$. For simplicity we assume all the linear transformations are invertible. 
The variable $S_j$ models the latent source of signal that is common to the subset of observations $\{U_i| i \in Q_j\}$. The matrix $A^{(i,j)}$ models the transformation of latent signal $S_j$ in view $i$. In order for the model to be identifiable, it is necessary that all the subsets $Q_j$'s are distinct (otherwise the latent sources with identical $Q_j$ can be collapsed into one $S_j$). In the most general setting, we have a latent signal that is uniquely associated with every subset of observations. In this case, $p = 2^k -1$ and $\{Q_j\}$ corresponds to all the non-empty subsets of $[k]$. In some settings, only specific subset of views $U_i$ share common signals and $\{Q_j\}$ can be a small set. We measure the complexity of the set system using the following notion:

\begin{definition}[$L$-distinguishable] We say a set system $\{Q_j\}$ is $L$-distinguishable, if for every set $Q_j$, there exists a subset $T\subset Q_j$ of size at most $L$ (called the distinguishing set) such that for any other set $Q_{j'} (j'\ne j)$, either $Q_j \subset Q_{j'}$ or $T\not\subset Q_{j'}$.\label{def:distinguishable}
\end{definition}

For example, the set system of the contrastive model is $\{\{1\}, \{1,2\}, \{2\}\}$ and it is 2-distinguishable. Intuitively, for any set $Q_j$ in the set system, there is a subset $T$ of size at most $L$ that distinguishes $Q_j$ from all the other sets (except the supersets of $Q_j$). We use Algorithm~\ref{alg:findlinearsimple} to recover all the linear transformations $A^{(i,j)}$ (for more details of the algorithm see Algorithm~\ref{alg:findlinear} in Appendix). Algorithm~\ref{alg:findlinearsimple} takes as input a set system $\{Q_j\}$ that captures our prior belief about how the datasets are related. When we don't have any prior belief, we can input the most general $\{Q_j\}$ of size $2^k-1$, which is $k$-distinguishable. The algorithm automatically determines if certain variable $S_j = 0$. In the algorithm, $\min(Q_j)$ is the smallest element of $Q_j$.


\begin{algorithm}
\begin{algorithmic}
\REQUIRE set system $\{Q_j\}$ that is $L$-distinguishable, $L+1$-th order moments
\REPEAT
\STATE Pick a set $Q_j$ that is not a subset of any remaining sets
\STATE Let $T = \{w_1,w_2,...,w_L\}$ be the distinguishing set for $Q_j$
\STATE Compute cumulants for all $i\in Q_j$: $M_i  = unfold(\cum_{L+1}(U_{w_1},...,U_{w_L}, U_i)$. 
\STATE If $M_{\min Q_j} = 0$ then the variable $S_j = 0$; continue the loop.
\STATE Let $A^{(i,j)} = (M_{\min Q_j}^\dag M_i)^\top$ for all $i\in Q_j$, $A^{(i,j)} = 0$ for all $i\not \in Q_j$.
\STATE Mark $Q_j$ as processed, subtract all the cumulants of $Q_j$.
\UNTIL{all sets are processed}
\end{algorithmic}
\caption{FindLinear}\label{alg:findlinearsimple}
\end{algorithm}

\begin{lemma}
\label{lem:findAgeneral}
Given observations $U_i$'s as defined in Equation~\ref{eq:model}, suppose the sets $Q_j$'s are $L$-distinguishable, all the unknown linear transformations $A^{(i,j)}$'s are invertible,  unfoldings $unfold(\cum_{L+1}(S_j))$ is either 0 (if $S_j=0$) or have full rank, then given the exact $L+1$-th order cumulants, Algorithm~\ref{alg:findlinearsimple} outputs all the correct linear transformations $A^{(i,j)}$ in time $\mbox{poly}(L!, (dk)^{L})$.
\end{lemma}

Once all the linear transformations $A^{(i,j)}$ are recovered, we follow the same strategy as in the contrastive analysis case \ref{sec:contrast}. 

\begin{theorem}
\label{lem:cumulantgeneral}
Under the same assumption as Lemma~\ref{lem:findAgeneral}, for any $t\ge L$ Algorithm~\ref{alg:cumulant} computes the correct $t$-th order cumulants for all the variables in time $\mbox{poly}((L+t)!, (dk)^{L+t})$. 
\end{theorem}

Note that in the most general case it is impossible to find cumulants with order $t < L$, because there can be many different variables $S_j$'s but not enough views. Both Algorithms \ref{alg:findlinearsimple} and \ref{alg:cumulant} are robust to noise, with sample complexity that depends polynomially on the relevant condition numbers, and exponential in the order of cumulant considered. 
For more details see Appendix~\ref{app:general}.

\subsection{Related models}
Independent component analysis (ICA)\cite{comon2010handbook} may appear similar to our model, but it is actually quite different. In ICA, let ${\bf s} = [s_1, ..., s_n]$ be a vector of latent sources, where $s_i$'s are \emph{one dimensional} independent, non-Gaussian random variables. There is an unknown mixture matrix ${\bf A}$ and the observations are ${\bf x = As}$. Given many samples ${\bf x^{(t)}}$, the goal is to deconvolve and recover each sample $s_i$. In our setting, each $s_i$ can be a \emph{high-dimensional} vector with complex correlations. It is information-theoretically not possible to deconvolve and recover the individual samples $s_i$. Instead we aim to learn the distribution $S_{i}(\theta)$ without having explicit samples from it. 

Another related model is canonical correlation analysis (CCA)\cite{hotelling1936relations}. The generative model interpretation of CCA is: there is a common signal $z\sim N(0,I)$, and view-specific signals $z^{(m)}\sim N(0,I)$. Each view $x^{(m)}$ is then sampled according to 
$N(A^{(m)}z + B^{(m)}z^{(m)}, \Sigma^{(m)})$, where $m$ index the view.
CCA is equivalent to maximum likelihood estimation of $A^{(m)}$ in this generative model. In our framework, CCA corresponds to the very restricted setting where $S_1, S_2, S_3$ are all Gaussians. \alg learns $S_1$ without making such parametric assumptions about $S_2$ and $S_3$. Moreover, using CCA, it is not clear how to learn the distribution $S_1$ if it is not orthogonal to the shared subspace $S_2$. In our experiments, we show that the naive approach of performing CCA (or kernel CCA) followed by taking the orthogonal projection leads to very poor performance. 
Factor analysis (FA)\cite{harman1976modern} also corresponds to a multivariate Gaussian model, and hence does not address the general problem that we solve. In FA, latent variables are sampled $z \sim N(0, I)$ and the observations are $x | z \sim N(\mu + \Lambda z, \Psi)$. 

A different notion of contrastive learning was introduced in \cite{zou13}. They focused on settings where there are two mixture models with overlapping mixture components. The method there applies only for Latent Dirichlet allocation and Hidden Markov Models and requires explicit parametric models for each component.

\section{Using Cumulants in learning applications}
\label{sec:apply}
The cumulant extraction techniques of Section~\ref{sec:cumulant} constructs unbiased estimators for the cumulants of $S_i$. 
In this section we show how to use the estimated cumulants/moments to perform maximum likelihood learning of $S_i$. For concreteness, we frame the discussion on the contrastive learning setting, where we want to learn $S_1$. For general \alg the method works when $L$ (see Definition~\ref{def:distinguishable}) is small or the distributions have specific relationship between lower and higher cumulants. 

\subsection{Method-of-Moments}\label{sec:mom} \alg recovers the cumulants of $S_1$, from which we can construct all the moments of $S_1$ in time $O((td)^t)$.  This makes it possible to directly combine \alg with any estimation algorithm based on the method-of-moments. Method-of-moments have numerous applications in machine learning. The simplest (and most commonly used) example is arguably {\em principal component analysis}, where we want to find the maximum variance directions in $S_1$. This is only related to the covariance matrix $\E[S_1S_1^\top]$. \alg removes the covariance due to $S_2$ and constructs an unbiased estimator of $\E[S_1S_1^\top]$, from which we can extract the top eigen-space.

The next simplest model is least squares regression (LSR). Suppose the distribution $S_1$ contains samples and labels $(X,Y) \in \R^d \times R$, and only the samples are corrupted by perturbations, i.e. $Y$ is independent of $S_2$. LSR tries to find a parameter $\beta$ that minimizes $\E[(Y-\beta^\top X)^2]$. The optimal solution again only depends on the moments of $(X,Y)$: $\beta^* = (\E[XX^\top])^{-1} \E[YX]$. Using the second-order cumulants/moments extracted from \alg, we can efficiently estimate 
$\beta^*$.

Method-of-moment estimators, especially together with tensor decomposition algorithms have been successfully applied to learning many latent variable models, including Mixture of Gaussians (GMM), Hidden Markov Model, Latent Dirichlet Allocation and many others (see \cite{anandkumar2012tensor}). \alg can be used in conjunction with all these methods. We'll consider learning GMM in Section~\ref{sec:expt}.

\subsection{Approximating Gradients}
\label{sec:gradient}
There are many machine learning models where it's not clear how to apply method-of-moments. Gradient descent (GD) and stochastic gradient descent (SGD) are general purpose techniques for parameter estimation across many models. Here we show how to combine \alg with gradient descent. 
The key idea is that the extracted cumulants/moments of $S_1$ forms a polynomial basis. If the gradient of the log-likelihood can be approximated by a low-degree polynomial in $S_1$, then the extracted cumulants from \alg can be used to approximate this gradient. 

Consider the general setting where we have a model $\mathcal{D}$ with parameter $\theta$, and for any sample $s_1$ the likelihood is $\mathcal{L}(\theta,s_1)$. The maximum likelihood estimator tries to find the parameter that maximizes the likelihood of observed samples:
$
\theta^* = \arg\max \E[\log \mathcal{L}(\theta,s_1)].
$
In many applications, this is solved using stochastic gradient descent, where we pick a random sample and move the current guess to the corresponding gradient direction:
$
\theta^{(t+1)} = \theta^{(t)} + \eta_t \nabla_\theta \log \mathcal{L}(\theta,s_1^{(t)}),
$
where $\eta_t$ is a step size and $s_1^{(t)}$ is the $t$-th sample. For convex functions this is known to converge to the optimal solution \cite{sgd}. Even for non-convex functions this is often used as a heuristic.

If the gradient of log-likelihood $\nabla_\theta \log \mathcal{L}(\theta,s_1)$ is a low degree polynomial in $s_1$, then using the lower order moments we can obtain an {\em unbiased} estimator for $\E[\nabla_\theta \log \mathcal{L}(\theta,S_1)]$ with bounded variance, which is sufficient for stochastic gradient to work. This is the case for linear least-squares regression, and its regularized forms using either $\ell_1$ or $\ell_2$ regularizer. 

In the case when log-likelihood is not a low degree polynomial in $S_1$, we approximate the gradient by a low degree polynomial, either through simple Taylor's expansion or other polynomial approximations (e.g. Chebyshev polynomials, see more in \cite{powell1981approximation}). This will give us a biased estimator for the gradient whose bias decreases with the degree we use. In general, when the (negative) log-likelihood function is strongly convex we can still hope to find an approximate solution:

\begin{lemma}
\label{lem:approxgrad}
Suppose the negative log-likelihood function $F(\theta) = -\E[\log \mathcal{L}(\theta,S_1)]$ is $\mu$-strongly convex and $H$-smooth, given an estimator $G(\theta)$ for the gradient such that $\|G(\theta) - \nabla F(\theta)\| \le \epsilon$, gradient descent using $G(\theta)$ with step size $\frac{1}{2H}$ converges to a solution $\theta$ such that $\|\theta - \theta_*\|^2 \le \frac{8\epsilon^2}{\mu^2}$.
\end{lemma}

When high degree polynomials are needed to approximate the gradient, our algorithm requires number of samples that grows exponentially in the degree. 

\paragraph{Logistic Regression}
We give a specific example to illustrate using \alg and low degree polynomials to simulate gradient descent. Consider the basic logistic regression setting, where the samples $s_1 = (x,y) \in \R^d \times \{0,1\}$, and the log-likelihood function is $\log \mathcal{L}(\theta, s_1) = \log \frac{e^{y \theta^\top x}}{1+e^{\theta^\top x}}$. The gradient of the log-likelihood is:
$
\nabla_\theta \log \mathcal{L}(\theta,s_1) = (y - \frac{e^{\theta^\top x}}{1+e^{\theta^\top x}}) x.
$

We can then approximate the function $\frac{e^{\theta^\top x}}{1+e^{\theta^\top x}}$ using a low degree polynomial in $\theta^\top x$. As an example, we use 3rd degree Chebychev:
$
\frac{e^{\theta^\top x}}{1+e^{\theta^\top x}} \approx  0.5 + 0.245\theta^\top x -  0.014(\theta^\top x)^3.
$
The gradient we take in each step is 
$$
\E[\nabla_\theta \log \mathcal{L}(\theta,S_1)] \approx \E[YX] - 0.5\E[X] - 0.245\E[X (\theta^\top X)] + 0.014\E[X (\theta^\top X)^3].
$$
To estimate this approximation, we only need quadratic terms $\E[X(\theta^\top X)]$ and a {\em projection} of the 4-th order moment $\E[X(\theta^\top X)^3 ]$. These terms are computed from the projected 2nd and 4-th order cumulants of $X$ that are extracted from the cumulants of $U$ and $V$ via Section~\ref{sec:cumulant}. Because of the projection these quantities are much easier to compute (in fact, they can be estimated in linear time).

\section{Experiments}
\label{sec:expt}

\begin{figure}[!htb]
\centering
\includegraphics[trim = {1cm 1cm 1cm 1cm}, clip,scale=0.75]{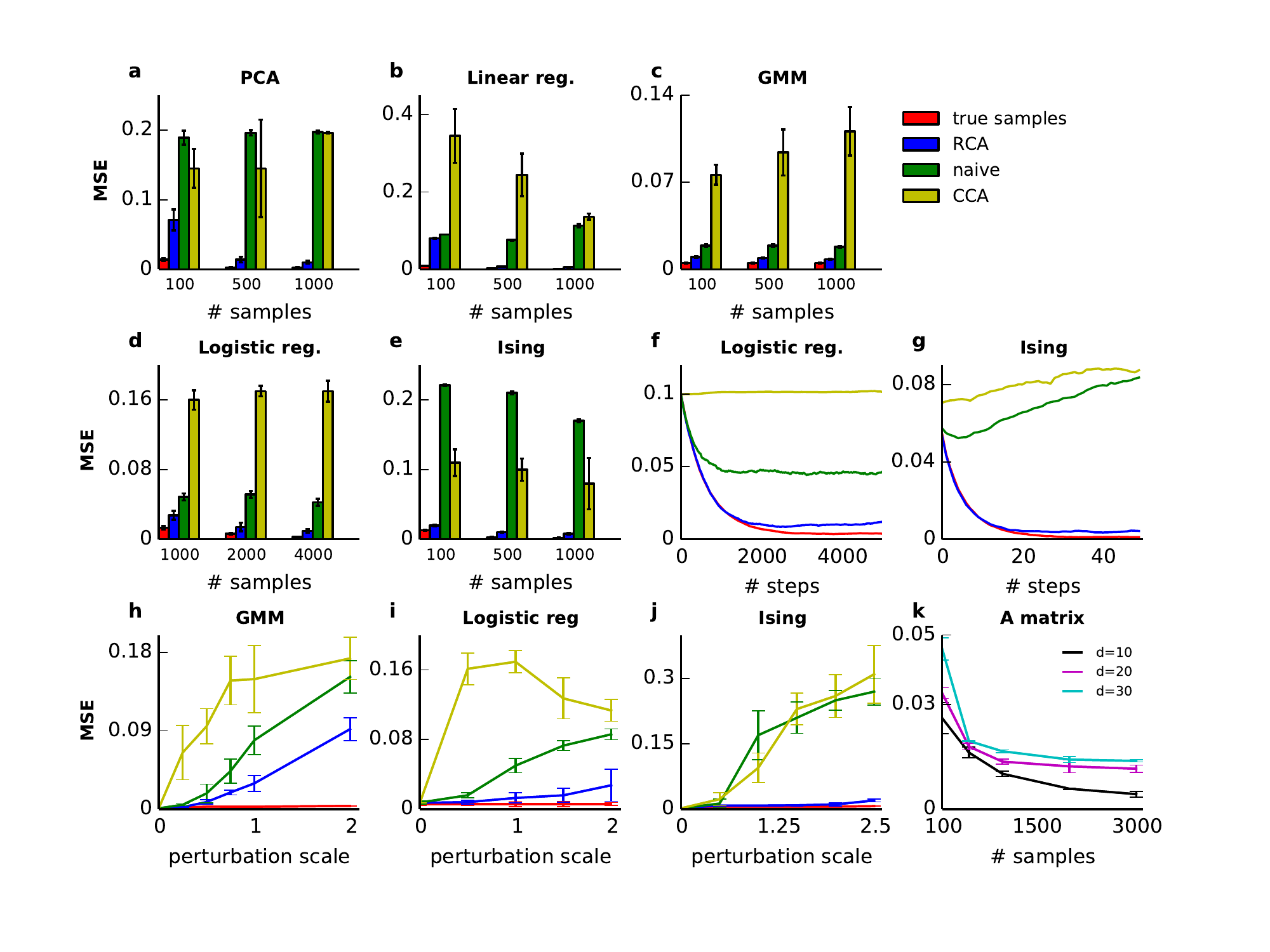}
\caption{All the $y$-axis indicate mean squared error (MSE). \textbf{a-e} shows the tradeoff between sample size and MSE for the four algorithms in each of the five applications. \textbf{f,g} shows the convergence rate of SGD for the logistic and Ising models. \textbf{h-j} shows the tradeoff between perturbation strength and MSE. \textbf{k} shows the inference accuracy of $A$. Error bars corresponds to standard deviation.}
\label{fig:expt1}
\end{figure}

In the experiments, we focus on the contrastive learning setting where we are given observations of $U = S_1 + S_2$ and $V = AS_2 + S_3$ and the goal is to estimate the parameters for the $S_1$ distribution. Our approach can also learn the shared component $S_2$ as well as $S_3$. We tested our method in five settings, where $S_1$ corresponds to: low rank Gaussian (PCA), linear regression, mixture of Gaussians (GMM), logistic regression and the Ising model. The first three settings illustrate combining \alg with method-of-moments and the latter two settings requires \alg with polynomial approximation to stochastic gradient descent. 
In each setting, we compared the following four algorithms:
\begin{enumerate}[topsep=0pt,itemsep=0pt]
\item The standard learning algorithm using the actual samples $s_1 \sim S_1(\theta)$ to learn the parameters $\theta$. This is the gold-standard, denoted as `true samples'.
\item Our contrastive \alg algorithm using paired samples from $U$ and $V$ to learn $S_1(\theta)$.
\item The naive approach that ignores $S_2$ and uses $U$ to learn $S_1(\theta)$ directly, denoted as `naive'.
\item First perform Canonical Correlation Analysis (CCA) on $U$ and $V$, and  project the samples from $U$ onto the subspace orthogonal to the canonical correlation subspace. Then learn $S_1$ from the projected samples of $U$. We denote this as `CCA'.  
\end{enumerate}
In all five settings, we let $S_3$ be sampled uniformly from $[-1,1]^{d}$, where $d$ is the appropriate dimension of $S_3$. The empirical results are robust to other choices of $S_3$ that we have tried, e.g. multivariate Gaussian or mixture of Gaussians.

\paragraph{Contrastive PCA.} $S_1$ was set to have a principal component along direction $v_1$, i.e. $s_1 \sim \N(0, v_1v_1^\top + \sigma^2 I)$. $S_2$ was sampled from $\mbox{Unif}([-1,1]^d) + v_2v_2^\top$ and $v_1$, $v_2$ are random unit vectors in $\R^d$. \alg constructs an unbiased estimator of $\E[S_1S_1^\top]$ from the samples of $U$ and $V$. We then report the top eigenvector of this estimator as the estimated $\hat{v}_1$. We evaluate each algorithm by the mean squared error (MSE) of the inferred $\hat{v}_1$ to the true $v_1$. 

\paragraph{Contrastive regression.} $S_1$ is the uniform distribution, $s_1 \sim \mbox{Unif}([-1, 1]^d)$ and $y = \beta^\top s_1 + \N(0, 1)$.  $S_2$ was sampled from $\mbox{Unif}([-1,1]^d) + v_2v_2^\top$ and $\beta$, $v_2$ are random unit vectors in $\R^d$. Our approach gives unbiased estimator of $\E[S_1S_1^\top]$ from which we estimate $\hat{\beta} = (\E[S_1S_1^\top])^{-1} \E[YS_1]$. All algorithms are evaluated by the MSE between the inferred $\hat{\beta}$ and the true $\beta$.  

\paragraph{Contrastive mixture of Gaussians.} $S_1$ is a mixture of $d$ spherical Gaussians in $\R^d$, $s_1 \sim \sum_{k = 1}^d \frac{1}{d} \N(\mu_k^{(1)}, \sigma^2)$. $S_2$ is also a mixture of spherical Gaussians, $s_2 \sim \sum_{k = 1}^d \frac{1}{d} \N(\mu_k^{(2)}, \sigma^2)$. \alg gives unbiased estimators of the third-order moment tensor, $\E[s_1 \otimes s_1 \otimes s_1]$. We then use the estimator in \cite{hsukakade} to get a low rank tensor whose components correspond to center vectors, and apply alternating minimization (see \cite{kolda2009tensor}) to infer $\hat{\mu}^{(1)}_k$. Algorithms are evaluated by the MSE between the inferred centers $\{\hat{\mu}^{(1)}_k\}$ and the true centers $\{\mu^{(1)}_k\}$.

\paragraph{Contrastive logistic regression.} Let $s_1 \sim \mbox{Unif}([-1, 1]^d)$ and $y = 1$ with probability $\frac{1}{1 + e^{-\beta^\top s_1}}$. $S_2$ was sampled from $\mbox{Unif}([-1,1]^d) + v_2v_2^\top$, and $\beta$, $v_2$ are unit vectors in $\R^d$. We use the 4-th order Chebychev polynomial approximation to the SGD of logistic regression as in Section~\ref{sec:gradient}.
Evaluation is the MSE error between the inferred $\hat{\beta}$ and the true $\beta$.

\paragraph{Contrastive Ising model.} Let $S_1$ be a mean-zero Ising model on $d$-by-$d$ grid with periodic boundary conditions. Each of the $d^2$ vertices are connected to four neighbors and can take on values $\{\pm 1\}$. The edge between vertices $i$ and $j$ is associated with a coupling $J_{ij} \sim \mbox{Unif}[-1,1]$. The state of the Ising model, $s_1$, has probability $\frac{1}{Z} e^{\sum_{(i,j)\in E} J_{ij}s_1(i)s_1(j)}$, where $Z$ is the partition function. We let $S_2$ also be a $d$-by-$d$ grid of spins where half of the spins are independent Bernoulli random variables and the other half are correlated, i.e. they are all 1 or all -1 with probability 0.5. We use composite likelihood to estimate the couplings $J_{ij}$ of $S_1$, which is asymptotically consistent with MLE of the true likelihood \cite{complike}. For the gold-standard baseline (which uses the true samples $s_1$), we use the exact gradient of the composite likelihood. For \alg, we used the 4-th order Taylor approximation to the gradient. Evaluation is the MSE between the true $J_{ij}$ and the estimated $\hat{J}_{ij}$. 

\paragraph{Results.} For the method-of-moment applications--PCA, linear regression, GMM--we used 10 dimensional samples for $U$ and $V$. The tradeoff between inference accuracy (measured in MSE) and sample size is shown in the top row of Figure~\ref{fig:expt1}. Even with just 100 samples, \alg performs significantly better than the naive approach and CCA. With 1000 samples, the accuracy of \alg approaches that of the algorithm using the true samples from $S_1$. It is interesting to note that projecting onto the subspace orthogonal to CCA can perform much worse than even the naive algorithm. In the linear regression setting, for example, when the signal of $S_2$ happens to align with $\beta$, the direction of prediction, projecting onto the subspace orthogonal to $S_2$ loses much of the predictive signal. 

In the SGD settings, we used a 10 dimensional logistic model and a 5-by-5 Ising model (50 $J_{ij}$ parameters to infer). \alg also performed substantially better than the two benchmarks (Figure~\ref{fig:expt1} d, e). In all the cases, the accuracy of \alg improved monotonically with increasing sample size. This was not the case for the Naive and CCA algorithms, which were unable to take advantage of larger data due to model-misspecification. In Figure~\ref{fig:expt1} f and g, we plot the learning trajectory of \alg over the SGD steps for representative runs of the algorithm with 1000 samples. \alg converges to the final state at a rate similar to the true-sample case. The residual error of \alg is due to the bias introduced by approximating the sigmoid with low-degree polynomial. When many samples are available, a higher-degree polynomial approximation can be used to reduce this bias. 

We also explored how the algorithms perform as the magnitude of the signal in $S_2$ is increased compared to $S_1$ (Figure~\ref{fig:expt1} h-j) with fixed 1000 samples. In these plots the $x$-axis measures the ratio of standard deviations of $S_2$ and $S_1$.
At close to 0, most of the signal of $U$ comes from $S_1$, and all the algorithms are fairly accurate. As the strength of the perturbation increases, \alg performs significantly better than the benchmarks, especially in the Ising model. Finally we empirically explored the sample complexity of the subroutine to recover the $A$ matrix from the 4th order cumulants. Figure~\ref{fig:expt1} k shows the MSE between the true $A$ (sampled $\sim \mbox{Unif}[-1,1]^{d \times d}$) and the inferred $\hat{A}$ as a function of the sample size. Even with 1000 samples, we can obtain reasonable estimates of $A\in \R^{30\times 30}$.   

\paragraph{Biomarkers experiment.} We applied \alg to a real dataset of DNA methylation biomarkers. Twenty biomarkers (10 test and 10 control) measured the DNA methylation level (a real number between 0 and 1) at twenty genomic loci across 686 individuals \cite{zou14ewasher}. Each individual was associated with a binary disease status $Y$. Logistic regression on the ten test biomarkers was used to determine the weight vector, $\beta$, which quantifies the contribution of the methylation at each of these ten locus to the disease risk. The other ten independent loci are control markers. Getting accurate estimates for the values of $\beta$ is important for understanding the biological roles of these loci. In this dataset, all the samples were measured on one platform, leading to relatively accurate estimate of $\beta$. In many cases samples are collected from multiple facilities (or by different labs). 
We simulated this within our \alg framework. We let $S_1$ be the original data matrix of the ten test markers across the 686 samples. We let $S_3$ be the original data matrix of the ten control markers in these same samples. We modeled $S_2$ as a mixture model, where samples are randomly assigned to different components that 
capture lab specific biases. The perturbed observations are $U = S_1 + S_2$ and $V = AS_2 + S_3$, i.e. $U$ and $V$ simulate the measurements for the test and control markers, respectively, when the true signal has been perturbed by this mixtures distribution of lab biases. We assume that we can only access $U$ and $V$ and do not know $S_2$, i.e. where each sample is generated. Running logistic regression directly on $U$ and the phenotype $Y$ obtained a MSE of 0.24 (std 0.03) between the inferred $\hat{\beta}$ and the true $\beta$ measured from directly regressing $S_1$ on $Y$. Directly using CCA also introduce significant errors with MSE of 0.25 (std 0.02). Using all the control markers as covariates in the logistic regression, the MSE of the test markers' $\beta$ was 0.14 (std 0.03). In general, adding $V$ as covariates to the regression can eliminate $S_2$ at the expense of adding $S_3$, and can reduce accuracy when $S_3$ is larger than $S_2$. Using our \alg logistic regression on $U$ and $V$, we obtained significantly more accurate estimates of $\theta$, with MSE 0.1 (std 0.03). See Appendix for more analysis of this experiment. 

\clearpage
\bibliographystyle{plain}
\bibliography{arxiv}

\newpage
\appendix

\section{More Tensor and Cumulant Notations}

\label{app:tensor}

In this section we introduce the notations and basics for tensors and cumulants.

\paragraph{Matrix Notations} For a matrix $M \in \R^{n\times m}$ we use $\|M\|$ to denote its spectral norm $\sup_{\|x\| = 1} \|Mx\|$, $\|M\|_F$ to denote its Frobenius norm $\|M\|_F = \sqrt{\sum_{i,j} M_{i,j}^2}$, and $\sigma_{min}(M)$ to denote its smallest singular value.

When $n\ge m$ and the matrix $M$ has full column rank, we use $M^\dag$ to denote its Moore-Penrose pseudoinverse which in particular satisfy $M^\dag M = I$.

We also sometimes use the Kronecker product of matrices, for $A\in \R^{m\times n}$ and $B\in \R^{p\times q}$, $A\otimes B$ is a matrix in $\R^{mp \times nq}$ that has the following block structure:

$$
A\otimes B = \left[\begin{array}{cccc}A_{1,1}B & A_{1,2} B & \cdots & A_{1,n} B \\
A_{2,1} B & A_{2,2}B & \cdots &A_{2,n}B \\ \vdots & \vdots & & \vdots \\ A_{m,1} B & A_{m,2} B &\cdots & A_{m,n} B \end{array}\right]
$$

The singular values of $A\otimes B$ is just the product of singular values of $A$ and $B$.

\paragraph{Tensor Notations} A tensor $T \in \R^{d^t}$ is a $t$-dimensional array, and is frequently used to represent higher order moments or cumulants. We index the elements in the tensor using a $t$-tuple $(i_1,i_2,...,i_t)\in [d]^t$. The entries of tensor product $[u_1\otimes u_2 \otimes\cdots u_t]_{(i_1,i_2,...,i_t)}$ is simply the product of corresponding entries $\prod_{j=1}^t u_j(i_j)$. We use $u^{\otimes t}$ to denote $u\otimes u\otimes \cdots \otimes u$ $t$ times.

For a distribution $X\in \R^d$, the $t$-th order moment is a tensor $\E[X^{\otimes 4}]$, whose $(i_1,i_2,...,i_t)$-th entry is equal to $\E[X_{i_1}X_{i_2}\cdots X_{i_t}]$. Later we shall see cumulants can also be conveniently represented as tensors.

A tensor can be viewed as a multi-linear form (just as a matrix $M$ can be viewed as a bilinear form $u^T M v$). For a tensor $T$ we define $T(M_1,M_2,\ldots, M_t)$ to be
$$
T(M_1,M_2,\ldots, M_t)_{(i_1,...,i_t)} = \sum_{(j_1,...,j_t)\in [d]^t} T_{(j_1,...,j_t)} \prod_{l = 1}^t M_l(j_l,i_l).
$$

This multi-linear form works well with the moment tensors, especially for matrices $M_1,...,M_t$ we always have
$$
\E[X^{\otimes t}](M_1,...,M_t) = \E[(M_1^\top X)\otimes (M_2^\top X)\otimes \cdots \otimes (M_t^\top X)].
$$

Often to simplify operations tensors are unfolded to become matrices. There can be many ways to unfold a tensor, but in this paper we mostly use a particular unfolding which makes the tensor into a $\R^{d^{t-1} \times d}$ matrix:
$$
unfold(T)_{(i_1,...,i_{t-1}), i_t} = T_{(i_1,...,i_t)}.
$$

Similar to matrices, we also define the Frobenius norm of tensors to be the $\ell_2$ norm of all its entries, in particular
$$
\|T\|_F = \|unfold(T)\|_F = \sqrt{\sum_{i_1,...,i_t} T_{(i_1,...,i_t)}^2}.
$$

\paragraph{Cumulants} Cumulants provide an alternative way to describe the lower order correlations of a random variable. Unlike moments, cumulants have the nice property that the cumulant of sum of independent random variables equals to the sum of cumulants. Formally, for a random variable $X \in \R$ the cumulant is defined to be the coefficients of the cumulant generating function $\log \E[e^{tX}]$ ($\cum_t(X)$ is just $t!$ times the coefficient in front of $X^t$).  When the variables are different the cross-cumulants (similar to covariance) can similarly  be defined, and it can be computed as:

\begin{equation}
\cum_t(X_1,...,X_t)  = \sum_{\pi}(|\pi|-1)!(-1)^{|\pi|-1}\prod_{B\in \pi}\E[\prod_{i\in B} X_i]. \label{eq:cumcompute}
\end{equation}

In this formula, $\pi$ is enumerated over all partitions of $[t]$, $|\pi|$ is the number of parts in partition and $B$ runs  through the list of all parts.

Similarly, it is possible to define cumulants for multivariate distributions. For random variable $X\in \R^d$ $\cum_t(X)_{(i_1,...,i_t)} = \cum_t(X_{i_1},...,X_{i_t})$. This cross cumulant can be computed in a similar way as Equation (\ref{eq:cumcompute}), however the products should be replaced by tensor products and the ordering of coordinates is important when doing the tensor product.

\begin{fact}\label{lem:cumulant:app}
Suppose $X_1,...,X_t$ are random variables in $\R^d$. The $t$-th order cumulant $\cum_t(X_1,...,X_t)$ is a tensor in $\R^{d^t}$ that have the following properties:
\begin{enumerate}
\item(Independence) If $(X_1,...,X_t)$ and $(Y_1,...,Y_t)$ are independent, then $\cum_t(X_1+Y_1,...,X_t+Y_t) = \cum_t(X_1,...,X_t) + \cum_t(Y_1,...,Y_t)$.
\item(Linearity) $\cum_t(M_1^\top X_1,...,M_t^\top X_t) = \cum_t(X_1,...,X_t)(M_1,...,M_t)$.
\item(Relation to Moments) The $t$-th order cumulant is a polynomial over the first $t$-th order moments. Similarly the $t$-th order moment is a polynomial over the first $t$-th order cumulants. Further both polynomials can be computed in $O(t!)$ time. Converting between first $t$-th order moments and cumulants for $d$-dimensional variables takes $O((td)^d)$ time.
\end{enumerate}
\end{fact}

Intuitively, cumulants can measure how correlated two distributions are. The simplest case is $\cum_2(X,Y)$ which is equal to the covariance $\E[(X-\E[X])(Y-\E[Y])]$, and is 0 only if the two variables are not correlated in second order. For more detailed introductions to cumulants see books like \cite{kenney1954mathematics}.


\section{Details for Section~\ref{sec:cumulant}}

\label{sec:app:cumulant}

In this section, we prove the equations and algorithms in Section~\ref{sec:cumulant} indeed compute the desirable quantity, and further we give sample complexity bounds.

\subsection{Contrastive Learning}
\label{app:contrastive}
We first prove Equation (\ref{eq:findA}) computes the correct linear transformation.

\begin{lemma}
[Lemma \ref{lem:contrastivelinear} restated]
Suppose the unfolding of the 4-th order cumulant $unfold(\cum_4(AS_2,S_2,S_2,S_2))$ has full rank, given the exact cumulants $\cum_4(V,U,U,U)$ and $\cum_4(V,U,U,V)$, Equation (\ref{eq:findA}) finds the correct linear transformation in time $O(d^5)$. 
\end{lemma}

\begin{proof}
Since $U = S_1 + S_2$ and $V = AS_2+S_3$, we know 
\begin{align*}
\cum_4(V,U,U,U) & = \cum_4(AS_2+S_3,S_1+S_2,S_1+S_2,S_1+S_2)\\
& = \cum_4(0, S_1, S_1, S_1) + \cum_4(AS_2, S_2, S_2, S_2)+\cum_4(S_3, 0, 0, 0) \\
& = \cum_4(AS_2, S_2, S_2, S_2).
\end{align*}
Here the second step uses the fact that cumulants are additive for independent variables, and third step uses the linearity of cumulants.

Similarly, we know $\cum_4(V,U,U,V) = \cum_4(AS_2, S_2, S_2, AS_2) = \cum_4(AS_2, S_2, S_2, S_2) (I, I, I, A^\top)$.

For the unfoldings of these cumulants, we have
$$
unfold(cum_4(V,U,U,V)) = unfold(cum_4(V,U,U,U)) A^\top.
$$

Therefore when $unfold(cum_4(V,U,U,V))$ has full rank we can compute $A$ using pseudo-inverse.

For the running time, the main computation is a pseudo-inverse and a matrix product for $d^3\times d$ matrices, both take $O(d^5)$ time.
\end{proof}

Next we show given the linear transformation, it is possible to estimate the cumulants using Equations (\ref{eq:s1simple} - \ref{eq:s3simple}). In fact, we can also avoid computing the cross-cumulants and work with just the cumulants of variables:
\begin{eqnarray}
\cum_t(S_1) & = &  \cum_t(U) - \frac{\cum_t(U+A^{-1} V) - \cum_t(U) - \cum_t(A^{-1}V)}{2^t - 2}, \label{eq:s1}\\
\cum_t(S_2) & = & \frac{\cum_t(U+A^{-1} V) - \cum_t(U) - \cum_t(A^{-1}V)}{2^t - 2},\label{eq:s2}\\
\cum_t(S_3) & = & \cum_t(V) - \frac{\cum_t(AU+V) - \cum_t(AU) - \cum_t(V)}{2^t - 2}.\label{eq:s3}
\end{eqnarray}

\begin{theorem}[Theorem~\ref{lem:cumulantsimple} restated]
For all $t > 1$, Equations (\ref{eq:s1simple})-(\ref{eq:s3simple}) or (\ref{eq:s1})-(\ref{eq:s3}) compute the correct cumulants for $S_1,S_2,S_3$ in time $O((td)^{t+2})$. Moreover, if $V$ has dimension higher than $U$ and $A$ has full column rank, replacing $A^{-1}$ by $A^\dag$ still gives correct cumulants.
\end{theorem}

\begin{proof}
The proof of Equations (\ref{eq:s1simple})-(\ref{eq:s3simple}) is very similar to the previous lemma. Note that
\begin{align*}
\cum_t(U,U,...,U,A^{-1}V) & = \cum_t(S_1+S_2,...,S_1+S_2,S_2+A^{-1} S_3)\\
& = \cum_t(S_1, ..., S_1, 0) + \cum_t(S_2, S_2, S_2, S_2)+\cum_t(0, ..., 0, A^{-1}S_3) \\
& = \cum_t(S_2).
\end{align*}
So we have Equation (\ref{eq:s2simple}), and using the fact that $\cum_t(U) = \cum_t(S_1)+\cum_t(S_2)$ we get Equation (\ref{eq:s1simple}). Equation (\ref{eq:s3simple}) follows similarly.

In order to get Equations (\ref{eq:s1})-(\ref{eq:s3}), first note that by the linearity of cumulants, we can write $\cum_t(U+A^{-1}V)$ as the sum of $2^t$ terms:
$$
\cum_t(U+A^{-1}V) = \sum_{z\in \{0,1\}^t} \cum_t(z_1U + (1-z_1) A^{-1}V,z_2U + (1-z_2) A^{-1}V,...,z_tU + (1-z_t) A^{-1}V).
$$
Among all these terms, one is equal to $\cum_t(U)$, one is equal to $\cum_t(A^{-1}V)$, and all the other $2^t-2$ terms are cross-cumulants that involve both $U$ and $V$. Since $S_2$ is the only variable that appears in both $U$ and $A^{-1}V$, all the $2^t-2$ terms are equal to $\cum_t(S_2)$, therefore we have Equation (\ref{eq:s2}). Equation (\ref{eq:s1}) again follows from the fact that $\cum_t(U) = \cum_t(S_1)+\cum_t(S_2)$, and Equation (\ref{eq:s3}) is very similar.

The moreover part follows by directly replacing $A^{-1}$ with $A^\dag$ in the above argument. Note that in this case we can still find $A$ because $unfold(cum_4(AS_2,S_2,S_2,S_2))$ still has full rank as long as $unfold(cum_4(S_2,S_2,S_2,S_2))$ has full rank.

For running time, the main bottleneck is computing the cumulants (which takes $O((td)^t)$ time), and then applying the matrix $A$ to the cumulants (which takes $O(d^{t+2})$ time).
\end{proof}

Finally, we show the equations are robust under sampling noise. For that we use the following bounds on cumulants

\begin{fact}(\cite{bulinskii1975bounds})
For any cross-cumulant $\cum_t(U_1,...,U_t)$, if all the variables have bounded norm $\|U_i\| \le R$, then the cumulant has Frobenius norm bounded by $(tR)^R$.
\end{fact}

In practice we use $k$-statistics \cite{rose2002mathematical} to estimate the cumulants, the standard deviation of $k$-statistics is bounded by a similar formula.

\begin{lemma}
Suppose the distributions $S_1,S_2,S_3$ have bounded radius $R$, the $4$-th order cumulant $unfold(\cum_4(V,U,U,U))$ has smallest singular value $\sigma_4$, matrix $A$ has smallest singular value $\sigma_A$ and $\|A\| \ge 1$, given $4$-th order cumulants that are $\epsilon$-close in Frobenius norm (and $\epsilon \ll R^4$), the linear transformation $A$ is recovered with accuracy $\epsilon\|A\|^2 R^4/\sigma^2_4$. Given $t$-th order cross-cumulants of $U$,$V$ that are $\epsilon_t$-close in Frobenius norm, the cumulants of $S_1$ can be computed with accuracy $O\left(\frac{\epsilon\|A\|^3R^4(tR)^t}{\sigma_4^2\sigma_A^2} + \frac{\epsilon_t}{\sigma}\right)$ using (\ref{eq:s1simple}). In particular, to estimate the cumulants of $S_1$ with accuracy $\eta$ the number of samples required is $ \Omega((tR)^{2t} \|A\|^{10} R^{16}/\sigma_4^4\sigma_A^4\eta^2)$.
\end{lemma}

\begin{proof}
First we show the algorithms are robust under perturbation. For that we need the fact that any $4$-th order cross-cumulant with bounded variables always have Frobenius norm of order at most $O(R^4)$. As a corollary we know the cross-cumulant $\cum_4(V,U,U,U)$ has norm at most $O(\|A\|R^4)$ and $\cum_4(V,U,U,V)$ has norm at most $O(\|A\|^2 R^4)$ Let $\hat{M}$ be the noisy version of $M = unfold(\cum_4(V,U,U,U))$, by assumption and by standard matrix perturbation bounds, we know $\|\hat{M}^\dag - M^\dag\|_F \le O(\epsilon/\sigma_4^2)$. On the other hand, let $\hat{N}$ be the noisy version of $N = unfold(\cum_4(V,U,U,V))$, we know $\|\hat{N}-N\|_F \le \epsilon$, therefore
$$
\| \hat{M}^\dag \hat{N} - M^\dag N\|_F \le O(\|\hat{M}^\dag - M^\dag\|\|N\|_F + \|M\|\|\hat{N}-N\|_F) \le O(\epsilon\|A\|^2R^4/\sigma_4^2).
$$

For computing the $t$-th order cumulant, the main source of error is applying $A^{-1}$ to the cross cumulant $\cum_t(U,...,U,V)$ to get $\cum_t(U,...,U,A^{-1}V)$, as we don't have the matrix $A$ exactly. Since the norm cross-cumulant is always bounded by $(tR)^t\|A\|$, we know when $\epsilon$ is small enough the error is roughly (ignoring lower order terms)
$$
\|\hat{A}^{-1} - A^{-1}\|\|\cum_t(U,...,U,V)\|_F + \|A^{-1}\|\|\hat{\cum_t}(U,...,U,V) - \cum_t(U,...,U,V)\|_F
$$
which is bounded by $$O\left(\frac{\epsilon\|A\|^3R^4(tR)^t}{\sigma_4^2\sigma_A^2} + \frac{\epsilon_t}{\sigma_A}\right).$$

Also, by the variance bounds for cumulants we know with $Z$ samples, $\epsilon_t \le (tR)^t\|A\|/\sqrt{Z}$ and $\epsilon \le O(R^4\|A\|^2/\sqrt{Z})$, therefore when $Z = \Omega((tR)^{2t} \|A\|^{10} R^{16}/\sigma_4^4\sigma_A^4\eta^2)$, the estimation of $S_1$ has desirable error.
\end{proof}

\subsection{Rich component analysis}
\label{app:general}
We first give the algorithm for computing the linear transformations and then show it computes the correct quantities.
\begin{algorithm}
\begin{algorithmic}
\REQUIRE set system $\{Q_j\}$ that is $L$-distinguishable, $L+1$-th order moments
\REPEAT
\STATE Pick a set $Q_j$ that is not a subset of any remaining sets
\STATE Let $T = \{w_1,w_2,...,w_L\}$ be the distinguishing set for $Q_j$
\STATE Compute cumulants for all $i\in Q_j$: \begin{align*}
M_i & = unfold(\cum_{L+1}(U_{w_1},...,U_{w_L}, U_i)\\
 & -\sum_{l:Q_j\subset Q_l}\cum_{L+1}(S_l)((A^{(w_1,l)})^\top,...,(A^{(w_L,l)})^\top, (A^{(i,l)})^\top))
\end{align*} 
\STATE If $M_{\min Q_j} = 0$ (or $\sigma_{min}(M_{\min Q_j})$ is too small), then $S_j = 0$; continue the loop.
\STATE Let $A^{(i,j)} = (M_{\min Q_i}^\dag M_i)^\top$ for all $i\in Q_j$, $A^{(i,j)} = 0$ for all $i\not \in Q_j$.
\STATE Mark $Q_j$ as processed, and let \begin{align*}
\cum_{L+1}(S_j) &= \cum_{L+1}((A^{(w_1,j)})^{-1}U_{w_1},...,(A^{(w_L,j)})^{-1}U_{w_L}, U_{\min Q_j})\\
 & -\sum_{l:Q_j\subset Q_l}\cum_{L+1}(S_l)((A^{(w_1,l)})^\top(A^{(w_1,j)})^{-\top},...,(A^{(w_L,l)})^\top(A^{(w_L,j)})^{-\top}, (A^{(\min Q_j, l)})^\top).
\end{align*}
\UNTIL{all sets are processed}
\end{algorithmic}
\caption{FindLinear}\label{alg:findlinear}
\end{algorithm}

The main idea behind this algorithm is that Since we know the sets are $L$-distinguishable, if we start from {\em maximal} set $Q_j$, there must be a distinguishing subset of size $L$ that is only contained in $Q_j$. Similar to the contrastive setting, if we consider a cross-cumulant that contains all the variables in this distinguishing set, then the resulting cumulant must only depend on this particular variable $S_j$. Further, using different last variable in the cross-cumulants (similar to using $\cum_4(V,U,U,U)$ and $\cum_4(V,U,U,V)$) and exploit the linearity of the cumulants, we can recover the linear transformations.

Note that without loss of generality we can assume $A^{(\min(Q_j),j)} = I$, because otherwise we can replace $S_j$ with the distribution $S_j' = A^{(\min(Q_j),j)} S_j$.

\begin{lemma}
[Lemma \ref{lem:findAgeneral} restated]
Given observations $U_i$'s as defined in Equation~\ref{eq:model}, suppose the sets $Q_j$'s are $L$-distinguishable, all the unknown linear transformations $A^{(i,j)}$'s are invertible,  unfoldings $unfold(cum_{L+1}(S_j))$ is either $0$ (when $S_j=0$) or  have full rank, then given the exact $L+1$-th order cumulants, Algorithm~\ref{alg:findlinear} outputs the correct linear transformations $A^{(i,j)}$ in time $\mbox{poly}(L!, (dk)^{L})$.
\end{lemma}

\begin{proof}
We prove this by induction using the following hypothesis:

For all the processed variables, Algorithm~\ref{alg:findlinear} finds the correct linear transformations $A^{(i,j)}$ and cumulant $\cum_{L+1}(S_j)$.

This hypothesis is clearly true at the beginning of the algorithm (as no variables are processed). We now show the algorithm will compute the correct quantities for the next variable $S_j$.

By the algorithm, we know the set $Q_j$ is not a subset of any remaining sets, and $T$ is a distinguishing set. Therefore, for other remaining set, we know it cannot contain all the elements in $T$. Therefore, by linearity and additivity of cumulants, we know $\cum_{L+1} (U_{w_1},...,U_{w_t}, U_i) (i\in Q_j)$ will only depend on the variable $S_j$ and some of the previously processed variables. In particular,
\begin{align*}
\cum_{L+1} (U_{w_1},...,U_{w_t}, U_i)  &= \cum_{L+1}(A^{(w_1,j)}S_j,A^{(w_2,j)}S_j,...,A^{(w_L,j)}S_j,A^{(j,i)} S_j) \\ &\quad + \sum_{l:Q_j\subset Q_l} \cum_{L+1}(A^{(w_1,l)}S_l,A^{(w_2,l)}S_l,...,A^{(w_L,l)}S_l,A^{(i,l)} S_l).
\end{align*}

By induction hypothesis, we have already processed all the other terms related to $Q_l$ ($l\ne j$ and $Q_j\subset Q_l$), so we have the correct cumulants $\cum_{L+1}(S_l)$ and linear transformations $A^{(i,l)}$'s. Those terms will be subtracted out during the algorithm. Therefore we know 
\begin{align*}
M_i & = unfold(\cum_{L+1}(A^{(w_1,j)}S_j,A^{(w_2,j)}S_j,...,A^{(w_t,j)}S_j,A^{(i,j)} S_j)) \\
& = unfold(\cum_{L+1}(A^{(w_1,j)}S_j,A^{(w_2,j)}S_j,...,A^{(w_t,j)}S_j,S_j)) (A^{(i,j)})^\top
\end{align*}
In particular $M_{\min Q_j} = unfold(\cum_{L+1}(A^{(w_1,j)}S_j,A^{(w_2,j)}S_j,...,A^{(w_t,j)}S_j,S_j) )$. Therefore, the algorithm computes the correct linear transformations if the matrix $M_{\min Q_j}$ has full rank.

The fact that $M_{\min Q_j}$ has full rank is implied by assumptions, because we can write this matrix as
$$
unfold(\cum_{L+1}(A^{(w_1,j)}S_j,A^{(w_2,j)}S_j,...,A^{(w_t,j)}S_j,S_j)) = (A^{(w_1,j)}\otimes \cdots \otimes A^{(w_L,j)}) unfold(\cum_{L+1}(S_j)).
$$
Here $\otimes$ is the {\em Kronecker product} of matrices, and it is well-known that the Kronecker product of invertible matrices are still invertible. Since $unfold(\cum_{L+1}(S_j))$ is either 0 or has full rank by assumption, we know we can either detect there is no component corresponding to set $Q_j$, or have a matrix $M_{\min Q_j}$ with full rank. In the latter case the correctness of the $L+1$-th order cumulant calculation then simply follows from the linearity of cumulants.

Finally, we estimate the running time of the algorithm. Computing any cumulant can be done in $\mbox{poly}(L!, d^L)$ time. Finding the distinguishing set (by exhaustive search) takes no more than $\mbox{poly}(k^L)$ time. The algorithm runs in at most $p \le 2^L$ iterations, each iteration computes a small number of cumulants and does small number of linear-algebraic calculations (which are all poly in $(kd)^L$), so the total running time is at most $\mbox{poly}(L!, (kd)^L)$
\end{proof}

Now we are ready to give the algorithm for computing cumulants and prove that it works.

\begin{algorithm}
\begin{algorithmic}
\REQUIRE set system $\{Q_j\}$ that is $L$-distinguishable, order $t \ge L$
\ENSURE $t$-th order cumulant for all the variables
\STATE Apply Algorithm~\ref{alg:findlinear} to find $A^{(i,j)}$'s, remove all sets whose variables do not appear.
\REPEAT
\STATE Pick a set $Q_j$ that is not a subset of any remaining sets
\STATE Let $T = \{w_1,w_2,...,w_L\}$ be the distinguishing set for $Q_j$, let $w_{L+1} = w_{L+2} =\cdots = w_{t} = w_L$.
\STATE Mark $Q_j$ as processed, let
\begin{align*}
\cum_{t}(S_j) & = \cum_{t}((A^{(w_1,j)})^{-1}U_{w_1},...,(A^{(w_{t},j)})^{-1}U_{w_{t}})\\
& - \sum_{l:Q_j\subset Q_l}\cum_{t}(S_l)((A^{(w_1,l)})^\top(A^{(w_1,j)})^{-\top},...,(A^{(w_{t},l)})^\top(A^{(w_{t},j)})^{-\top})
\end{align*}
\UNTIL{all sets are processed}
\end{algorithmic}
\caption{ComputeCumulant}\label{alg:cumulant}
\end{algorithm}

\begin{theorem}
[Theorem \ref{lem:cumulantgeneral} restated]
Under the same assumption as Lemma~\ref{lem:findAgeneral}, for any $t\ge L$ Algorithm~\ref{alg:cumulant} computes the correct $t$-th order cumulants for all the variables in time $\mbox{poly}((L+t)!, (dk)^{L+t})$. 
\end{theorem}

\begin{proof}
The proof of this lemma is very similar to the previous one. Again we prove the lemma by induction, under the following induction hypothesis:

For all the processed variables, Algorithm~\ref{alg:findlinear} finds the correct cumulant $\cum_{t}(S_j)$.

This is clearly true before the main loop. We now show that the algorithm successfully compute the cumulant of the next variable.

Similar as before, since $w_1,...,w_{t}$ contains all the elements of a distinguishing set $T$, we know

\begin{align*}
&\quad \cum_{t}((A^{(w_1,j)})^{-1}U_{w_1},...,(A^{(w_{t},j)})^{-1}U_{w_{t}}) \\&= \cum_{t}(S_j) + \sum_{l:Q_j\subset Q_l} \cum_{t}((A^{(w_1,j)})^{-1}A^{(w_1,l)}S_l,(A^{(w_2,j)})^{-1},A^{(w_2,l)}S_l,...,(A^{(w_{t},j)})^{-1}A^{(w_{t},l)}S_l) \\
& = \cum_{t}(S_j) + \sum_{l:Q_j\subset Q_l}\cum_{t}(S_l)((A^{(w_1,l)})^\top(A^{(w_1,j)})^{-\top},...,(A^{(w_{t},l)})^\top(A^{(w_{t},j)})^{-\top}).
\end{align*}

By induction hypothesis all the other terms are computed in previous iterations of the algorithm, so they are subtracted out. Therefore we get the first term which is equal to $\cum_{t}(S_j)$.
\end{proof}

Finally we prove the sample complexity bounds.

\begin{lemma}
Suppose the distributions $S_j$'s have bounded radius $R$, the $L+1$-th order cumulant $unfold(\cum_{L+1}(S_j))$ has smallest singular value $\sigma_\kappa$, nonzero matrices $A^{(i,j)}$ has smallest singular value $\sigma_A$ and spectral norm at most $\|A\|$. Also, suppose the longest chain of subsets $Q_{j_1}\subset Q_{j_2}\subset\cdots \subset Q_{j_q}$ has length $q$.
Given $L+1$-th order cumulants that are $\epsilon$-close in Frobenius norm, the linear transformation $A$ is recovered with accuracy $\epsilon(pLR\|A\|/\sigma_A\sigma_\kappa)^{O(qL)}$. Given $t$-th order cumulants that are $\epsilon$-close in Frobenius norm, the cumulants of $S_j$ can be computed with accuracy $\epsilon(pLR\|A\|/\sigma_A\sigma_\kappa)^{O(qL+qt)}$. In particular, to estimate the cumulants of $S_1$ with accuracy $\eta$ the number of samples required is $\Omega((pLR\|A\|/\sigma_A\sigma_\kappa)^{O(qL+qt)}/\eta^2)$.
\end{lemma}

\begin{proof}
We prove this by induction. For each variable $S_j$, let depth $q_j$ be the length of the longest chain such that $Q_j \subset Q_{j_1} \subset \cdots \subset Q_{q_{j}-1}$. We shall prove that the $A^{(i,j)}$'s are recovered with accuracy $\epsilon_{A,q_j} = O(\epsilon (p\|A\|^{2L+2}(LR)^{t+1}/\sigma_A^{2L} \sigma_\kappa^2)^{q_j-1}\|A\|^{t+1} (tR)^{L+1}/\sigma_A^{2L} \sigma_\kappa^2 )$ and the $L+1$-th cumulant is recovered with accuracy $\epsilon_{\kappa,q_j} = O(L\epsilon_{A,q_j}\|M_{\min Q_j}\|_F/\sigma_A^2)$.

First we show the base case, when $q_j = 1$ and therefore there is no other set that contains this set. In this case, $A^{(i,j)}$ is just equal to $(M_{\min Q_j}^\dag M_i)^\top$ where the $M$'s are the unfoldings of $L+1$-th order cross-cumulants (so we have them with accuracy $\epsilon$). By standard matrix perturbation bounds we know the error is bounded by $\epsilon\|M_i\|_F/\sigma_{min}(M_{\min Q_j})^2$ and we just need to bound the smallest singular value and Frobenius norm for the $M$'s. For $M_{\min Q_j}$, we know it is equal to a linear transformation of the unfolding of $\cum_{L+1}(S_j)$, therefore $\sigma_{min}(M_{\min Q_j}) \ge \sigma_A^L \sigma_c$. Similarly we have $\|M_i\|_F \le O(\|A\|^{L+1} (LR)^{L+1})$. Therefore $\epsilon_{A,1} = O(\epsilon \|A\|^{L+1} (LR)^{L+1}/\sigma_A^{2L} \sigma_\kappa^2)$. When we compute the $L+1$-th order cumulant, the dominating term is applying the inverses of the $A$ matrices we estimated, and we know $\epsilon_{\kappa,1} = O(L\epsilon_{A,1}\|M_{\min Q_j}\|_F/\sigma_A^2) = O(\epsilon L \|A\|^{2L+1} (LR)^{2L+2}/\sigma_A^{2L+2} \sigma_\kappa^2)$.

Suppose we have shown this for all of $Q_j$'s with small depth $q_j \le u$. For a set $Q_j$ with $q_j = u+1$, when we compute the matrices $M$ we need to subtract the cumulants of the previously computed variables. The number of such variables is at most $p$, and each variable has an additional error of $O(\epsilon_{\kappa,q_j-1} \|A\|^{L+1} + L\epsilon_{A,q_j-1} \|A\|^L (LR)^{L+1}) \le O(\epsilon_{\kappa,q_j-1} \|A\|^{L+1})$. This ($O(\epsilon_{\kappa,q_j-1} p\|A\|^{L+1})$) is our new error in estimating the cumulants. Therefore, by the same argument we have $\epsilon_{A,q_j} =  O(\epsilon_{\kappa,q_j-1} p \|A\|^{2L+2} (LR)^{L+1}/\sigma_A^{2L} \sigma_\kappa^2) = O(\epsilon (p\|A\|^{2L+2}(LR)^{L+1}/\sigma_A^{2L} \sigma_\kappa^2)^{q_j-1}\|A\|^{L+1} (LR)^{L+1}/\sigma_A^{2L} \sigma_\kappa^2 )$

The rest of the proof follows from very similar induction on Algorithm~\ref{alg:cumulant}.
\end{proof}
\section{Details for Section~\ref{sec:apply}}

In this section we prove Lemma~\ref{lem:approxgrad}, which shows for a strongly convex function, given a biased estimator for the gradient we can still hope to get close to its optimal solution.

\begin{lemma} [Lemma~\ref{lem:approxgrad} restated]
Suppose the negative log-likelihood function $F(\theta) = -\E[\log \mathcal{H}(\theta,S_1)]$ is $\mu$-strongly convex and $H$-smooth, given an estimator $G(\theta)$ for the gradient such that $\|G(\theta) - \nabla F(\theta)\| \le \epsilon$, gradient descent using $G(\theta)$ with step size $\frac{1}{2H}$ converges to a solution $\theta$ such that $\|\theta - \theta_*\|^2 \le \frac{8\epsilon^2}{\mu^2}$.
\end{lemma}

Before proving this lemma we first introduce basic definitions for strongly convex functions.

\begin{definition}[$\mu$-strongly convex] A function $F(\theta)$ (whose second order derivatives exist) is $\mu$-strongly convex if for any two points $\theta, \tau$ we have
$$
F(\theta) \ge F(\tau) + \langle \nabla F(\tau), \theta - \tau\rangle + \frac{\mu}{2} \|\theta - \tau\|^2.
$$
\end{definition}

\begin{definition}[$H$-smooth] A function $F(\theta)$ (whose second order derivatives exist) is $H$-smooth if for any two points $\theta, \tau$ we have
$$
F(\theta) \le F(\tau) + \langle \nabla F(\tau), \theta - \tau\rangle + \frac{H}{2} \|\theta - \tau\|^2.
$$
\end{definition}

The proof of Lemma~\ref{lem:approxgrad} mostly follows from the approximate gradient framework in \cite{dictionary}. For completeness we also give the proof here.

\begin{proof}
Let $\theta^*$ be the optimal point. First, by $\mu$-strongly convexity we know
$$
\langle \nabla F(\theta), \theta - \theta^* \rangle \ge F(\theta) - F(\theta^*) + \frac{\mu}{2}\|\theta - \theta_*\|^2.
$$
On the other hand, by $H$-smoothness we know
$$
F(\theta^*) \le \min_\eta F(\theta-\eta \nabla F(\theta)) \le \min_\eta F(x) - \eta \|\nabla F(\theta)\|^2 + \frac{H\eta^2}{2} \|\nabla F(\theta)\|^2 = F(x) - \frac{1}{2H}\|\nabla F(\theta)\|^2.
$$
Therefore
\begin{equation}
\langle \nabla F(\theta), \theta - \theta^* \rangle \ge \frac{1}{2H} \|\nabla F(\theta)\|^2 + \frac{\mu}{2}\|\theta - \theta_*\|^2. \label{eq:strongconvex}
\end{equation}

Now we prove even when the gradient $G(\theta)$ is only an approximation, the above equation still holds approximately.

\begin{claim}
$$
\langle G(\theta), \theta - \theta^*\rangle \ge \frac{1}{4H} \|G(\theta)\|^2 + \frac{\mu}{4} \|\theta - \theta_*\|^2 - 2\epsilon^2/\mu.
$$
\end{claim}

\begin{proof}
We know
\begin{align*}
\langle G(\theta), \theta - \theta^*\rangle & = \langle \nabla F(\theta),\theta-\theta_*\rangle + \langle G(\theta) - \nabla F(\theta),\theta-\theta_*\rangle\\
& \ge \langle \nabla F(\theta),\theta-\theta_*\rangle - \epsilon \|\theta-\theta_*\| 
\\ & \ge \langle \nabla F(\theta),\theta-\theta_*\rangle - \frac{\mu}{4}\|\theta-\theta_*\|^2 - \frac{\epsilon^2}{\mu}.
\end{align*}

Also, $\|G(\theta)\|^2 \le 2\|G(\theta) - \nabla F(\theta)\|^2 + 2\|\nabla F(\theta)\|^2$. Using these two inequalities in Equation (\ref{eq:strongconvex}), we get
\begin{align*}
\langle G(\theta), \theta - \theta^*\rangle & \ge \langle \nabla F(\theta),\theta-\theta_*\rangle - \frac{\mu}{4}\|\theta-\theta_*\|^2 - \frac{\epsilon^2}{\mu} \\
& \ge \frac{1}{2H}\|\nabla F(\theta)\|^2 + \frac{\mu}{4}\|\theta-\theta_*\|^2 - \frac{\epsilon^2}{\mu}\\
& \ge \frac{1}{4H}\|\nabla F(\theta)\|^2 + \frac{\mu}{4}\|\theta-\theta_*\|^2 - \frac{\epsilon^2}{2H} - \frac{\epsilon^2}{\mu} \\
& \ge \frac{1}{4H}\|\nabla F(\theta)\|^2 + \frac{\mu}{4}\|\theta-\theta_*\|^2 -  \frac{2\epsilon^2}{\mu}
\end{align*}
\end{proof}

The above Claim essentially matches the $(\alpha,\beta,\epsilon)$-approximate condition in \cite{dictionary}. Now suppose the update rule is
$$
\theta^{(t+1)} = \theta^{(t)} - \frac{1}{2H} G(\theta^{(t)}).
$$
We can then prove convergence result:
\begin{claim}
$$
\|\theta^{(t)} - \theta^*\|^2 \le (1-\frac{\mu}{4H})^t \|\theta^{(0)} - \theta_*\|^2 + \frac{8\epsilon^2}{\mu^2}.
$$
\end{claim}
\begin{proof}
We prove this by induction. Assume this is true for step $t$ (the base case $t=0$ is trivial), then for the next step we have
\begin{align*}
\|\theta^{(t+1)} - \theta^*\|^2 & = \|\theta^{(t)} - \theta^*\|^2 - \frac{1}{H} \langle G(\theta^{(t)}), \theta^{(t)}-\theta^*\rangle + \frac{1}{4H^2} \|G(\theta^{(t)})\|^2 \\
& = \|\theta^{(t)} - \theta^*\|^2 - \frac{1}{2H} (2\langle G(\theta^{(t)}), \theta^{(t)} \theta^*\rangle- \frac{1}{2H}\|G(\theta^{(t)})\|^2 )\\
& \le \|\theta^{(t)} - \theta^*\|^2 - \frac{1}{2H} ( \frac{\mu}{2}\|\theta^{(t)} - \theta^*\|^2 - \frac{4\epsilon^2}{\mu}) \\
& \le (1-\frac{\mu}{4H}) \|\theta^{(t)} - \theta^*\|^2 + \frac{2\epsilon^2}{\mu H}.
\end{align*}

Substituting in the bound for $\|\theta^{(t)} - \theta^*\|^2$ we get the exact claim.
\end{proof}

Therefore by carefully choosing the step size gradient descent quickly converges to a nearby point (in fact similar argument works as long as the learning rate is upper bounded by $\frac{1}{2H}$). Similar arguments can be proved for stochastic gradient with a small enough step size (depending on the variance).
\end{proof}
\section{Details for Section~\ref{sec:expt}}

\paragraph{Ising model inference.}
Let $\theta \equiv \{J_{ij}\}$, the composite log-likelihood $l_{cl}$ of this Ising model can be written as  

\[
l_{cl}(\theta) = \E_S\left[\sum_{i \in \mbox{Vert}} \log \frac{P(S| \theta)}{P(S, S(i) = 1 | \theta) + P(S, S(i) = -1 | \theta)}\right].
\]
The gradient of the composite log-likelihood with respect to a particular  $J_{ij}$ is
\begin{eqnarray}
\bigtriangledown_{J_{ij}}l_{cl} &=& \E_S\left[\frac{2S(i)S(j)}{1+\mbox{exp}(2S(i)\sum_{k \in \mbox{neigh}(i)}J_{ik}S(k) )} +  \frac{2S(i)S(j)}{1 + \mbox{exp} (2S(j) \sum_{k \in \mbox{neigh}(j)J_{jk}S(k)})} \right] \nonumber \\
&\approx & \E_S\left[ 2S(i)S(j) - S(i)^2S(j)\sum_{k \in \mbox{neigh}(i)}J_{ik} S(k) - S(i)S(j)^2\sum_{k \in \mbox{neigh}(j)}J_{jk} S(k) \right] \nonumber
\end{eqnarray}
where we have used the 4-th order Taylor expansion. We then use the 2nd and 4th order cumulant tensors from $U$ and $V$ to obtain unbiased estimator of terms $\E_S[S(i)S(j)]$ and $\E_S[S(i)^2S(j)S(k)]$. This gives the approximate gradient used in SGD. In the experiments, we used batch size of 100 samples to approximate each step of the gradient. 

\paragraph{Biomarkers experiment.} In the simulation for having samples from multiple facilities, we get two views of the data, where $U = S_1+S_2$ and $V = S_2' + S_3$. Here $S_1$ represents the values for test markers, $S_3$ represents the values for control markers, and $(S_2,S_2')$ jointly represents the perturbation caused by different labs.

In our set up, we assume the samples come from two different labs, each lab has a bias on all the 20 markers (we use $(p^1,q^1), (p^2,q^2)\in \R^{10}\times \R^{10}$ to denote the biases). That is
$$
(S_2,S_2') = \left\{ \begin{array}{rl} (p^1,q^1) & \mbox{sample from lab 1} \\(p^2,q^2) & \mbox{sample from lab 2} \end{array}\right.
$$

In this case, when the vectors $p,q$ are in general positions, it is easy to see that there is a rank-2 matrix $A$ such that $S_2' = AS_2$. In particular, let $P\in \R^{10\times 2}$ be the matrix whose columns are $p^1,p^2$, $Q\in \R^{10\times 2}$ be the matrix whose columns are $q^1,q^2$, then we know $A = QP^\dag$.

Note that this does not fit directly in our framework as the distribution $S_2$ has low rank, and therefore the 4-th order cumulant cannot have full column rank. However, we can consider $S_2 = PX_2$ and $S_2' = QX_2$ (where $X_2 = (1,0)$ for samples from lab 1 and $X_2 = (0,1)$ for samples from lab 2). We show that the algorithm still makes sense in this setting.

Let $W = unfold(\cum_4(QX_2,PX_2,PX_2, X_2)) \in \R^{1000\times 2}$, by the linearity of cumulants, we know
\begin{align*}
M_1 = unfold(\cum_4(V,U,U,U)) = WP^\top, \\
M_2 = unfold(\cum_4(V,U,U,V)) = WQ^\top.
\end{align*}

In this case, $M_1$ does not have full column rank, so the usual definition of pseudo-inverse does not work. However, we can write $P = ZR$ where $Z\in \R^{10\times 2}$ is an orthonormal matrix ($Z^\top Z = I$), and $R\in \R^{2\times 2}$ and then hope to find $M_1^\dag$ such that $M_1^\dag M_1 = ZZ^\top$ (this is possible because we can let $M_1^\dag := Z(WR^\top)^\dag$).

When we use this definition of pseudo-inverse, it is easy to check that $(M_1^\dag M_2)^\top = QR^{-1} Z^\top = QP^\dag$, therefore our algorithm can still recover the correct rank-2 $A$ matrix.

\end{document}